\documentclass{article}

\usepackage[preprint,main]{neurips_2026}
\usepackage[utf8]{inputenc} 
\usepackage[T1]{fontenc}    
\usepackage{hyperref}       
\usepackage{url}            
\usepackage{booktabs}       
\usepackage{amsfonts}       
\usepackage{nicefrac}       
\usepackage{microtype}      

\usepackage{multicol}
\usepackage{amsmath, amsfonts, amssymb, amsthm}
\usepackage{bm}
\usepackage[inline]{enumitem}
\usepackage{mathtools}
\usepackage{graphicx}
\usepackage{longtable,tabularx}
\usepackage{placeins} 
\usepackage{float}
\usepackage{multirow}
\usepackage{bbm}
\usepackage[table, dvipsnames]{xcolor}
\usepackage{threeparttable}
\usepackage{balance}
\usepackage[ruled,algo2e]{algorithm2e}
\usepackage{algpseudocode}
\usepackage{adjustbox}
\usepackage{booktabs}
\usepackage{bm}
\usepackage{etoolbox}
\usepackage{microtype}
\usepackage[title]{appendix}
\usepackage{units}
\usepackage{cleveref}
\usepackage{xspace}
\usepackage{tcolorbox}
\usepackage{wrapfig}
\usepackage{caption}
\usepackage{thmtools} 
\usepackage{thm-restate}
\usepackage{tensor}

\usepackage{pgfplots}
\usepackage{subcaption}
\pgfplotsset{compat=1.18}

\usepackage{booktabs}

\DeclareDocumentEnvironment{example}{}{\noindent\textbf{Running example:}\itshape}{}

\usepackage{ifthen}
\newboolean{include-notes}
\newboolean{include-new}
\newboolean{include-remove}
\setboolean{include-notes}{true}
\setboolean{include-new}{false}
\setboolean{include-remove}{false}

\usepackage[dvipsnames]{xcolor}
\usepackage[normalem]{ulem}
\providecommand{\justin}[1]{\ifthenelse{\boolean{include-notes}}{\textcolor{orange}{\textbf{Jaime:} #1}}{}}

\providecommand{\princeton}[1]{\ifthenelse{\boolean{include-notes}}{\textcolor{orange}{#1}}{}}

\providecommand{\p}[1]{\smallskip \noindent \textbf{{#1}.}}

\providecommand*{\algname}{{EpicWorldModel}\xspace}

\usepackage{graphicx}

\title{\algname: Exploration-driven Planning with Latent World Models}

\author{
Bowen Feng$^{1}$\thanks{Equal contribution.} \quad
Julian Ost$^{1}$\footnotemark[2] \quad
May Mei$^{1}$ \quad
Anirudha Majumdar$^{1}$ \quad
Felix Heide$^{1,2}$ \\[4pt]
$^{1}$Princeton University
\qquad
$^{2}$Torc Robotics
}

\begin{document}
\setcounter{footnote}{1}
\maketitle

\begin{abstract}
  Latent world models based on Joint-Embedding Predictive Architecture (JEPA) are deterministic by design. 
  While successful in fully observable scenarios, this paradigm breaks down when past observations and actions lead to multiple plausible future possibilities, e.g., due to occlusion.
  We introduce \textit{\algname}, a framework to train stochastic JEPAs for environments and tasks with inherent uncertainty under partially observability. 
  We jointly train the \textit{\algname} predictor with its latent representation space to directly predict multiple potential future states using a flow-matching objective, when the goal-relevant scene content is absent from the conditioning history.
  We show that flow predictive variance, motivated by its relation to an upper bound on predictive entropy, serves as a useful exploration guidance for planning. 
  By incorporating this uncertainty signal into Cross-Entropy Method (CEM)-based planning, our approach balances goal-reaching with exploration of uncertain regions where occluded goals are most likely to be located. 
  We demonstrate the effectiveness of \algname through a series of latent planning experiments with the best or on-par performance across tasks, showing up to $22\%$ empirical improvement in success rate over LeWorldModel.
\end{abstract}

\section{Introduction}
%
Learning world models for latent-space planning has seen increased applications in sequential decision-making, enabling agents to learn compact latent dynamics models and plan directly in the learned representation space~\citep{maes2026leworldmodel,zhou2025dinowm,bardes2023v,assran2025v}.
Unlike model-free end-to-end visuomotor policy learning, latent-space planning with world models does not directly learn a policy that maps observations to actions.
Instead, it learns to compress high-dimensional sensory observations, such as visual inputs, into a structured latent representation and uses a learned dynamics model to perform model-based planning within this shared latent space.

Joint-Embedding Predictive Architecture (JEPA) models have emerged as a promising family of models for this task. These methods learn predictive representations by mapping observations into an embedding space and predicting future or masked target embeddings from context embeddings, rather than reconstructing raw pixels~\citep{lecun2022path}.
This principle has been instantiated in image representation learning through I-JEPA~\citep{assran2023self}, video representation learning through V-JEPA~\citep{bardes2023v, assran2025v}, and motion-aware predictive representation learning through MC-JEPA~\citep{bardes2023mc}.
Shown to be promising for representation learning, early world models suffered from collapse~\citep{lecun2022path} when trained without complex training objectives, regularization through exponential moving averages and stop-gradients, the correct hyper-parameters~\citep{sobal2025stress} or pre-trained encoders~\citep{zhou2025dinowm}.
More recently, LeWorldModel~\citep{maes2026leworldmodel} stabilizes JEPA world model training with a simple yet effective regularization of the representation space to an isotropic Gaussian distribution. 

However, a key limitation of existing JEPA-based world models is that their prediction objective is often a deterministic single-target regression in the latent space.
This formulation can be effective when the future state is largely determined by the historical observation context.
The contrary is the case when the conditioned observation history does not contain sufficient information to uniquely determine the future.
For example, consider the scenario where a robot is opening a drawer. From the robot's perspective, any reasonable object could be inside the drawer. Further, the ground-truth future state could easily change over time, if the content of the drawer changed.
A world model predicting only a single future could lead to planning failure, if the predicted content of the drawer does not match the real world.
In such cases, multiple plausible futures may be consistent with the same observation history and the action chunk, but only a subset of them may lead toward the desired goal.
This raises a central question: \emph{
How can we learn a world model to predict useful future representations when the future is inherently ambiguous?
}
More broadly, we investigate latent-space planning that relies on predictive uncertainty to explore multiple plausible futures and select action sequences that effectively reach \emph{partially observed or occluded goals.}

To this end, we propose \algname: Exploration-based Planning with Latent World Model, a framework that trains a stochastic JEPA model to imagine multiple plausible futures in a jointly learned representation space, and plans under uncertainty, balancing goal-reaching and exploration.
At its core, \algname learns to predict future observations through sampling from the conditional distribution of possible futures given a set of input actions.
In practice, we model this stochastic prediction as a conditional flow-matching process~\citep{pooladian2023multisample} by predicting the average velocity~\citep{geng2025mean,geng2025improved} between a random distribution and multiple future representations. 
Past observations and actions serve as additional conditional signals. Encoder and predictor models are jointly learned~\citep{lecun2022path,maes2026leworldmodel} on pairs of past and future frames.

Our key insight is that when the model's uncertainty on the future frames is high, the flow-matching predictor will exhibit high variance in generated future states. Conversely, for future frames that have been observed in recent history, the world model will produce near-deterministic predictions with low variance.
We motivate this intuition through an information-theoretic framework, relating flow predictive variance to an upper bound on predictive entropy and serves as an exploration guidance.
By incorporating this uncertainty signal into Cross-Entropy Method (CEM) planning, our approach balances goal-reaching with exploration of uncertain regions where occluded or unseen future states are most likely located.
We support our hypothesis and approach through both theoretical analysis connecting flow variance to information gain and empirical evaluation on planning tasks.
Experimental results show that \algname achieves the best or on-par performance among all baselines, outperforming existing latent world models by up to $22\%$ across diverse partially observed tasks. Notably, \algname also reaches state-of-the-art performance on visually restricted maze navigation tasks that were previously challenging for latent world models.

\section{Related Works}
\p{World Models}
Recent advances in world modeling has extended state-space dynamics learning to high-dimensional representations such as RGB images~\citep{ha2018world,ha2018recurrent}. Although action-conditioned video generation models have emerged as a popular paradigm for modeling photorealistic evolution of the world~\citep{guo2025ctrl, quevedo2025worldgymworldmodelenvironment, zhu2024irasim, kim2026cosmos}, their pixel reconstruction objective tend to be computationally costly, deeming them impractical for real-time inference and control. In contrast, latent world models aim to learn useful representations of the world as opposed to individual pixels, in particular with joint-embedding predictive architectures (JEPAs)~\citep{lecun2022path}. 
Video JEPA models~\citep{bardes2024revisiting, assran2025v} are trained on internet-scale video data to capture temporal structure for future prediction. D-JEPA~\citep{chen2024djepa} incorporates diffusion and flow matching objectives into JEPA, extending applicability to image, video, and audio generation.
Conditioning on robot actions, models such as DINO-WM~\citep{zhou2025dinowm} and LeWorldModel~\citep{maes2026leworldmodel} leverage learned representations or enforce structure in the latent space, enabling direct application to robot planning and control.

\p{Flow Matching}
Early work on flow matching draws inspiration from Continuous normalizing flows (CNFs)~\citep{chen2018neural} originally introduced as a way to model ordinary differential equations (ODEs) from data pairs using black-box solvers.
CNFs modeled the underlying vector field $v_t$ of the ODE.
As a descendant continuation conditional flow matching (CFM)~\citep{lipman2022flow} methods learned such CNFs for optimal transport between two distributions by only using linear solvers and samples from both distributions
This simplicity enabled brought use across data domains.
The image generation task has been most successfully solved by denoising diffusion probabilistic models~\citep{song2020denoising} that stochastic differential equations (SDEs) before.
Flow matching offered an appealing alternative achieved by reformulating the denoising process as an ODE.
The continuous denoising path of CFMs in contrast to the inherent randomness in the path of SDEs resulted in less denoising steps, later further improved with the rectified flow objective~\citep{liu2022rectflow}.
Not only predominant in scaled image generation~\citep{esser2024scaling} other data domains such as 3D shapes~\citep{xiang2025structured}, molecules~\citep{du2024machine}, or robotic planning~\citep{black2024pi_0} have seen successful implementations of generative models using flow matching.
Most relevant to out work are video generation models~\citep{wan2025wan, hacohen2024ltx} that have been continuously extended towards unbounded  generation horizon~\citep{huang2025self,chen2024diffusion} and action conditioning ~
\citep{guo2025ctrl} rendering them useful for down-stream world modeling~\citep{lingbot-world}.
Unlike video-generation world models that predict pixels, \algname’s stochastic predictor follows the JEPA paradigm by learning an expressive representation space and using a flow ODE to predict future representations.

\p{Uncertainty-Aware Planning} 
%
Early work in uncertainty-aware planning grounded intrinsic motivation in prediction error within learned feature spaces~\citep{pathak2017curiosity} and modeled aleatoric and epistemic uncertainty through probabilistic ensemble dynamics for trajectory-sampled planning~\citep{chua2018pets}.
Plan2Explore~\citep{sekar2020planning} demonstrates that ensemble disagreement in latent dynamics models provides an effective intrinsic signal for directed exploration.
Subsequent work extends this paradigm to structured world models~\citep{sancaktar2022cee}, achieving zero-shot transfer to multi-object manipulation.
More recently, DreamerV3-XP~\citep{bierling2025dreamerxp} augmented the DreamerV3 model~\citep{hafner2023dreamerv3} with ensemble-based reward disagreement as an intrinsic exploration incentive.
\algname draws insight from the exploration literature, and uses a stochastic JEPA world model for both goal-reaching and curiosity signals.
\begin{figure}
    \centering
    \includegraphics[width=1.0\linewidth]{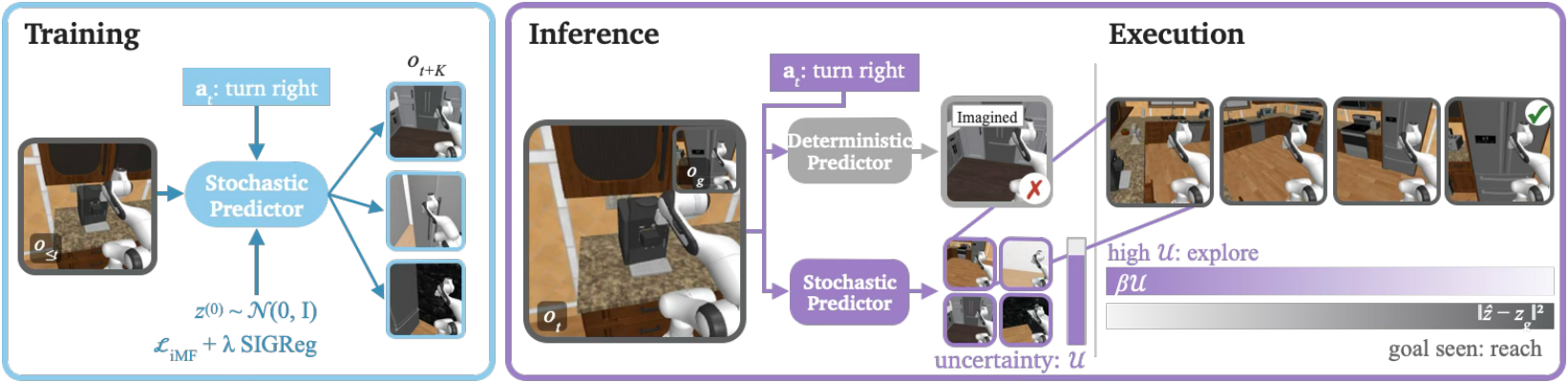}
    \caption{\textbf{A robot is asked to navigate to a fridge initially outside its field of view.} A deterministic world model memorizes the location of the fridge from training data and hallucinates it at inference, even when the true location differs. A stochastic model instead maintains a distribution over plausible futures, enabling uncertainty-aware planning. Dark gray framed images indicate observation inputs; colored frames indicate predicted futures. We only show four possibilities at $t$ and the results of a single choice at $t+1$. In practice observations and live in the jointly learned representation space.}
    \label{fig:robot_fridge}
\end{figure}
\section{\algname}
\label{sec:method}
We propose a stochastic latent world model for goal-conditioned planning under partial observability.
Instead of regressing to a single deterministic future embedding, our model learns a conditional distribution over future latent representations.
This allows planning over multiple plausible latent futures when critical scene is occluded or absent from the current observation history.

\subsection{Problem Formulation}
\label{sec:preliminaries}
We consider goal-conditioned planning in a partially observable Markov decision process (POMDP) $\mathcal{M}=(\mathcal{S},\mathcal{O},\mathcal{A},p,q,r,\gamma)$, where $\mathcal{S}$ is the hidden state space, $\mathcal{O}$ is the observation space, and $\mathcal{A}$ is the action space.
At time $t$, the environment is in an unobserved state $s_t\in\mathcal{S}$ and an observation $o_t\in\mathcal{O}$ is made according to $o_t\sim q(\cdot\mid s_t)$.
After the agent executes an action $a_t\in\mathcal{A}$, the environment evolves according to the transition rule $s_{t+1}\sim p(\cdot\mid s_t,a_t)$.
Since $s_t$ is not directly observed, the agent must plan from a finite observation-action history.

\subsection{Learning the Stochastic Latent Predictive World Model}
\label{sec:model}

We build upon latent world modeling frameworks~\citep{maes2026leworldmodel,sobal2025stress,sobal2022joint,zhou2025dinowm}-- more precisely JEPA-based models ~\cite{maes2026leworldmodel} -- which learn an encoder $\mathcal E_\theta$ that maps observations into a compact latent space and a predictor $\mathcal P_\varphi$ that models action-induced temporal evolution in this latent space.
In its standard form, the predictor is trained to regress a single future embedding, which corresponds to a deterministic approximation of the future latent state.

However, we argue that under partial observability, the future may not be uniquely determined by the current observation history and action, especially when uncertainty is present in the problem. 
As an example, consider a robot tasked to navigate to the fridge in a kitchen that is initially not within its field of view (\Cref{fig:robot_fridge}).
In vanilla JEPA, the world model will learn to ``memorize'' the location of the fridge as they were in the training dataset, e.g., to the robot's right. During inference, given a robot action of turning right, the world model will hallucinate a fridge in the future latent state, even if the fridge is actually to the robot's left~\citep{yin2025womap}.
This is especially challenging when the goal is absent from the conditioning frames. To model such prediction uncertainty, we replace the deterministic predictor with a conditional generative dynamics model.
Given a latent history and a candidate future action chunk,
we learn a latent world model that predicts a distribution $p_\varphi$
over future latent states conditioned on both the latent history and the candidate action chunk as
\begin{align}
    \text{Encoder } & \mathbf z_t = \mathcal{E}_\theta (o_t),\\
    \text{Action Embedder } & \mathbf{c}_{t} =\mathcal C_\psi (\mathbf a_{t:t+K-1})\\
    \text{Predictor } & \mathbf z_{t+1:t+K} \sim \mathcal P_\varphi(\mathbf{z}_{t+1:t+K}\mid \mathbf{z}_{t-H+1:t},\mathbf{c}_{t})
\end{align}
where $H$ is the history length and $K$ is the planning horizon.

\p{Conditional Flow Matching in Gaussian Latent Space}
To learn a $\mathcal P_{\varphi}$ that predicts a distribution over future latent states, we adopt the conditional flow matching (CFM) paradigm~\citep{lipman2022flow}. 
Flow matching learns a time-dependent velocity $v_\varphi$ of a vector field that transports samples from a source distribution, e.g., a Gaussian distribution $\mathcal N(0, I_d)$, to a target distribution.
In our case, the target distribution consists of latent future observations $\mathbf{z}_{t+1:t+K}$ given the latent history $\mathbf{z}_{t-H+1:t}$ and action chunk $\mathbf{c}_t$.

Let $\mathbf{z}^{(1)} = \mathbf{z}_{t+1:t+K}$ be the target future latent block and $\mathbf{z}^{(0)}\sim\mathcal{N}(0,I)$ be a Gaussian source sample of the same shape. 
We can now define a straight-line interpolation~\cite{liu2022rectflow} for flow step $\tau\in[0,1]$, as well as the corresponding conditional velocity target as
\begin{align}\label{eq:velocity_fm}
    \mathbf{z}^{(\tau)} &= (1-\tau)\mathbf{z}^{(0)} + \tau \mathbf{z}^{(1)}, \qquad
    v_{\mathrm{tgt},\tau}(\mathbf{z}^{(\tau)}\mid \mathbf{z}^{(1)}) = \mathbf{z}^{(1)} - \mathbf{z}^{(0)}.
\end{align}
Training $v_\varphi$ to match this per-sample target leads to an unbiased estimator, allowing for scalable training without marginalizing over all source-target pairs~\cite{lipman2022flow}.
This leads to the flow matching objective
\begin{equation}\label{eq:fm_objective}
    \mathcal{L}_{\mathrm{FM}}
    =
    \underset{\tau,\mathbf{z}^{(0)},(\mathbf{x}_{t},\mathbf{c}_{t},\mathbf{z}_{t+1:t+K})}{\mathbb E}
    \left\|
    v_\varphi(\mathbf{z}^{(\tau)},\tau,\mathbf{x}_{t},\mathbf{c}_{t})
    -
    \left(v_{\mathrm{tgt},\tau}(\mathbf{z}^{(\tau)}\mid \mathbf{z}^{(1)})\right)
    \right\|^2_2.
\end{equation}

\p{Average Velocity Prediction for 1-Step Inference}
Flow matching comes with an inherently higher complexity than deterministic prediction at inference time caused by a high number of function evaluations (NFE) in each step of the Euler sampling process $\mathbf{z}^{(\tau+\Delta \tau)} = \mathbf{z}^{(\tau)} + \Delta \tau\cdot
    v_\varphi$.

More recent models~\citep{lu2026one, geng2025improved, geng2025mean} characterize the underlying flow field by its average velocity
\begin{equation}\label{eq:meanflow}
u \left( \mathbf{z}^{\left(\tau\right)} , r, \tau \right) \triangleq \frac{1}{\left(s - r\right) }  \int_r^t v(z^{(\omega)}) d\omega    
\end{equation} between two timesteps $\tau$ and $r$ with $\tau \leq t$ instead of the instantaneous velocity $v ( \cdot, \tau)$ at a single step $s$.
After applying the derivative w.r.t. $s$ to Eq~\eqref{eq:meanflow} and following the rewriting in improved MeanFlow (iMF)~\cite{geng2025improved} we get rid of the integral and derive the instantaneous velocity $v$ in \Cref{eq:velocity_fm} from the average velocity as
\begin{equation}
    v (z^{(\tau)}, \cdot) = u (z^{(\tau)}, \cdot) + (\tau - r) \frac{d}{ds} u (z^{(\tau)}, \cdot).
\end{equation}
This yields the learning objective
\begin{align}
\mathcal{L}_{iMF} &= \mathbb{E}\left\|V_\varphi - v_{\mathrm{tgt}}\right\|_2^2, \\
V_\varphi &=  u_\varphi(\mathbf{z}^{(\tau)}, \mathbf{x}_t,\mathbf{c}_t, r, \tau) + (\tau - r) \cdot  \mathrm{JVP}_{\mathrm{sg}} \left( u_\varphi ; v_\varphi \right), \text{ with } v_\varphi(\mathbf{z}^{(\tau)}, \tau) \equiv u_\varphi(\mathbf{z}^{(\tau)}, \tau, \tau)
\end{align}


and $\mathrm{JVP}_{\mathrm{sg}}$ is the Jacobian-Vector product approximation of $\frac{d}{d\tau} u(\mathbf{z}^{(\tau)}) = \partial_{\mathbf{z}} u(\mathbf{z}^{(\tau)}) v (\mathbf{z}^{(\tau)}) + \partial_\tau u (\mathbf{z}^{(\tau)}) $ with a stop-gradient $\mathrm{sg}$ applied.
While the added complexity during training leads to higher complexity at training time, the inference time can be collapsed to a single function evaluation (1-NFE) by choosing $\tau=0$ and $r=1$.

\p{Isotropic Gaussian Latent Regularization}
We additionally regularize the latent embeddings to follow an isotropic Gaussian distribution using the Sketched Isotropic Gaussian Regularizer (SIGReg)~\citep{balestriero2025lejepa} in line with ~\cite{maes2026leworldmodel}. SIGReg leverages the Cramér-Wold theorem: a multivariate distribution matches a target if all its one-dimensional projections match that target. Concretely, latent embeddings $\mathbf{Z}$ are projected onto $M$ random unit-norm directions $u^{(m)} \in \mathcal{S}^{d-1}$, and the Epps-Pulley normality test statistic $T(\cdot)$ is applied to each projection $h^{(m)} = \mathbf{Z} u^{(m)}$
\begin{equation}
    \mathrm{SIGReg}(\mathbf{Z}) \coloneqq \frac{1}{M} \sum_{m=1}^{M} T(h^{(m)}).
\end{equation}
Because the Epps-Pulley statistic tests against $\mathcal{N}(0,1)$ on each projection, convergence of SIGReg implies $p(z_t) \approx \mathcal{N}(0, I_d)$ at each timestep.
The complete \algname training objective is given by
\begin{equation}\label{eq:sigreg}
    \mathcal{L}_{\text{\algname}} \coloneqq \mathcal{L}_{\mathrm{iMF}} + \lambda\, \mathrm{SIGReg}(\mathbf{Z}),
\end{equation}

jointly supervising stochastic prediction of futures and a well defined latent representation.

\p{Theoretical Properties}
The Gaussian latent geometry induced by SIGReg confers several useful properties on the flow matching objective. We state these as formal propositions and provide proofs in Appendix~\ref{app:proofs}.
 
\begin{restatable}[Source-Target Compatibility]{remark}{rem:compatibility}
\label{rem:compatibility}
SIGReg encourages the latent space marginal distribution $\mathbb P_{\mathbf Z_t}\to\mathcal{N}(0,I_d)$ at each timestep $t$. With flow matching, the source distribution $p_0=\mathcal{N}(0,I_d)$ matches the latent space marginal distribution, leading to integrations always beginning in regions well-supported by the data distribution. This can lead to shorter and better conditioned transport paths in flow matching.
\end{restatable}

In CFM, the same interpolated point $\mathbf z^{(\tau)}$ can be reached from many different $(\mathbf z^{(0)}, \mathbf z^{(1)})$ pairs. Given only $\mathbf z^{(\tau)}$, the network is asked to predict a target $v_{\mathrm{tgt},\tau}$ that is not unique, leading to irreducible gradient variance~\cite{pooladian2023multisample, tong2023improving}. However, since the latent space is encouraged to be $\mathcal N(0,I_d)$ (where $d$ is the dimensionality), we can show that the gradient noise is bounded and data-independent.

\begin{figure}
\vspace{-12pt}
    \centering
    \includegraphics[width=\linewidth]{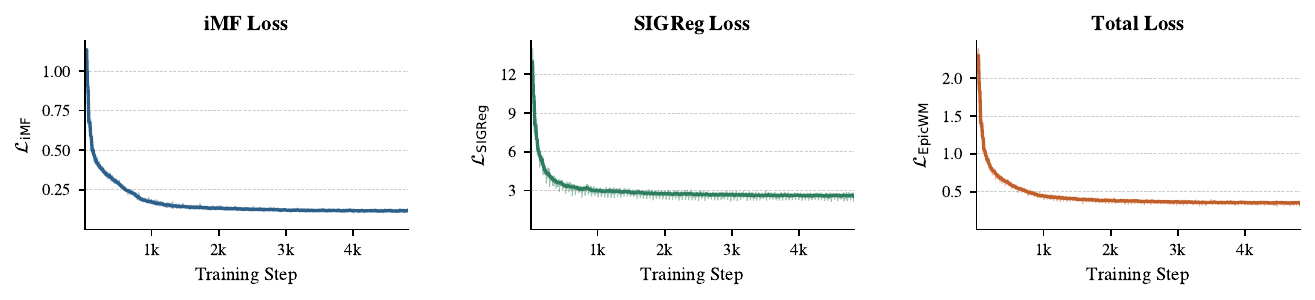}
    \caption{\textbf{Loss curves of \algname on the visual PointMaze environment.} The bold colored lines represents the smoothed loss and the faint colored lines the actual per step loss. Overall we observe minimal variance and smooth convergence of the encoder and stochastic predictor.}
    \label{fig:loss_curves}
\end{figure}
 
\begin{restatable}[Data-Independent Gradient Noise]{proposition}{propboundedgradient}
\label{prop:noise}
Given an isotropic Gaussian latent space with $\mathbb P_{\mathbf Z_t}=\mathcal{N}(0,I_d)$,
for every $\tau\in[0,1]$, the excess variance of the conditional velocity target is given by
\begin{equation}
\mathrm{Var}\left(v_{\mathrm{tgt},\tau}\mid \mathbf{z}^{(\tau)}\right)=\frac{1}{(1-\tau)^2+\tau^2}I_d.
\end{equation}
This quantity depends only on the latent dimensionality $d$ and the flow step $\tau$ (isotropic, $\propto I_d$ ). Therefore, gradient variance is bounded and independent of data, and does not grow as the latent representation becomes more expressive or the conditioning becomes more complex.
\end{restatable}

\Cref{fig:loss_curves} shows the validation loss across individual loss components in~\Cref{eq:meanflow,eq:sigreg} indicating smooth convergence as a result of~\Cref{prop:noise}.

\subsection{Uncertainty-Aware Latent Space Planning}
\label{sec:planning}

Since \algname outputs stochastic future predictions, we can approximate a group of sampled outputs as an ensemble, used to quantify the uncertainty in predicted future states. Specifically, given history $\mathbf x_t$ and action chunk embedding $\mathbf c_t$, we sample $N$ noise vectors ${\mathbf z}{^{(0)}}[i]\sim\mathcal N(0,I_d)$ and run forward pass with the predictor $\mathcal P_\varphi$ for each noise sample to acquire an ensemble of $N$ members, $\{{\hat{\mathbf{z}}}{^{(1)}}[i]\}_{i=1}^N$. Then, we use the disagreement between the ensemble members as a signal of uncertainty of the future state, computed as
\begin{equation}
    \mathcal U=\frac{1}{d}\sum_{j=1}^d\widehat{\text{Var}}_i({\hat{\mathbf z}}{_j^{(1)}}[i]),
\end{equation}
where $j$ is the latent feature dimension. We motivate this choice of this disagreement signal via the following proposition.

\begin{restatable}[Uncertainty as Conditional Entropy Upper Bound]{proposition}
{propuncertainty}
\label{prop:uncertainty}
Let $\Sigma_{\mathbf{c}_t}$ denote the true predictive covariance for an action chunk embedding $\mathbf{c}_t$. For a well-trained $\mathcal P_\varphi$, $\mathcal{U}$ is an unbiased estimator of the mean conditional variance
\begin{equation}
    \mathbb{E}[\mathcal{U}] \;=\; \frac{\mathrm{tr}(\Sigma_{\mathbf{c}_t})}{d}.
\end{equation}
By the maximum entropy theorem, the predictive entropy is upper-bounded by a function of $\mathcal{U}$ as
\begin{equation}\label{eq:bound}
    H(\mathbf{z}^{(1)}\mid\mathbf{x}_t,\mathbf{c}_t)
    \;\leq\;
    \frac{d}{2}\log\!\left(2\pi e\cdot\mathcal{U}\right).
\end{equation}
Furthermore, $\mathcal{U}$ is bounded within $[0,1]$, providing an interpretable absolute scale. $\mathcal{U}\approx 0$ indicates a confident, concentrated prediction, while $\mathcal{U}\approx 1$ indicates predictions as diffuse as the marginal prior.
The bound in Eq.~\ref{eq:bound} is tight only for $\Sigma_{c_t} \propto I_d$ (App. A, Eq. 31); for anisotropic covariances it can be loose, so we use $\mathcal{U}$ as a practical exploration bonus rather than an entropy estimate.
\end{restatable}

\p{Connection to Expected Information Gain}
Under partial observability, the agent must explore to reduce its uncertainty about unobserved future states. A principled exploration objective is to maximize the Expected Information Gain (EIG)~\citep{houlsby2011bayesian,gal2017deep}, i.e., the amount of information gained from taking an action chunk $\mathbf a_{t:t+K-1}$ and observing the outcome $\mathbf z$. We motivate our uncertainty measure $\mathcal U$ via EIG through the following remark.

\begin{restatable}[Connection to Expected Information Gain]{remark}
{remarkeig}
\label{rem:eig}
In exploration, we seek to gain more information about the future latent state $\mathbf z$. Given the flow matching source distribution, we have a prior over $\mathbf z^{(0)}$ given by $\mathcal N(0, I_d)$. We can interpret sampled pure noises as different ``seeds'' of the world. In the Bayesian active learning literature~\citep{houlsby2011bayesian,gal2017deep}, EIG measures mutual information between the seed and potential future latent states
\begin{equation}
    \mathrm{EIG}(\mathbf{c}_t)=I(\mathbf z^{(0)}; 
  {\mathbf z}^{(1)} | \mathbf x_t, \mathbf c_t) = H[{\mathbf{z}}^{(1)} \mid \mathbf{x}_t, \mathbf{c}_t] -             
  H[{\mathbf{z}}^{(1)} \mid \mathbf{x}_t, \mathbf{c}_t, \mathbf{z}^{(0)}].
\end{equation}
Since the flow ODE is deterministic given $\mathbf{z}^{(0)}$, the per-seed entropy $H[\hat{\mathbf{z}}^{(1)}\mid\mathbf{x}_t,\mathbf{c}_t,\mathbf{z}^{(0)}] =
0$. Therefore, 
\begin{equation}
    \mathrm{EIG}(\mathbf{c}_t)
    \;=\;
    H\!\left[{\mathbf{z}}^{(1)}\mid\mathbf{x}_t,\mathbf{c}_t\right]
    \;\lesssim\; \frac{d}{2}\log(2\pi e\cdot\mathcal{U}),
\end{equation}
where the inequality is from \Cref{prop:uncertainty}. Therefore, $\mathcal{U}$ is a proxy motivated by EIG, and maximizing $\mathcal{U}$ approximately maximizes an upper bound on EIG.
\end{restatable}

\p{Uncertainty-guided Planning Objective}
At inference time, we optimize over action chunk embeddings using Model Predictive Control (MPC) with the Cross-Entropy Method (CEM)
\begin{equation}
    \mathbf{c}^\star_{t}
    \;=\;
    \arg\min_{\mathbf{c}_{t}}
    \;\sum_i^{N}\!\left[\left\|{\hat{\mathbf{z}}}{_{t+K}}[i]
    - \mathbf z_g\right\|^2\right]
    \;-\;\beta\,\mathcal{U}(\mathbf{c}_{t}),
\end{equation}
where $\mathbf z_g=\mathcal{E}_\theta(o_g)$ is the goal embedding and $\beta\geq 0$ balances exploration and exploitation (goal-reaching).

\begin{figure}
\vspace{-12pt}
    \centering
    \includegraphics[width=1.0\linewidth]{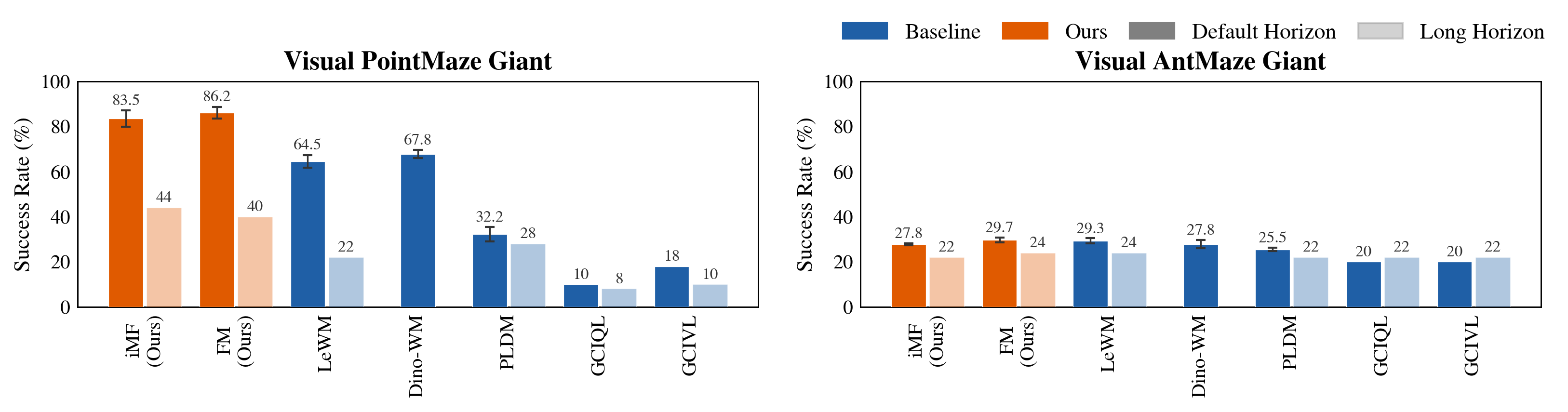}
    \caption{\textbf{Success rate comparison across benchmark environments. }FM and IMF represent \algname with the flow matching objectives introduced in~\Cref{eq:fm_objective} (FM) and~\Cref{eq:meanflow} (iMF). {\algname achieves the best performance in both the default- and long-horizon settings.}}
    \label{fig:success}
\end{figure}
\begin{figure}
\vspace{-8pt}
    \centering
    \includegraphics[width=\linewidth]{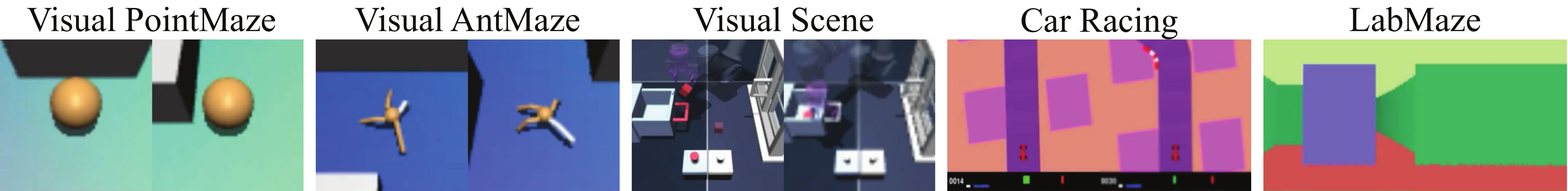}
    \caption{\textbf{Evaluation Tasks.} In each of the five tasks the most left image represents the initial observation $o_0$ and the right a goal observation. All tasks are POMDP evaluations in existing benchmarks. The first three tasks are available in OGbench~\cite{ogbench_park2025}, Car Racing is available in Gymnasium~\cite{towers2024gymnasium}, and LabMaze in DMLab~\cite{beattie2016deepmind}.}
    \label{fig:eval_tasks}
\end{figure}
\section{Experiments}



We evaluate \algname across Visual PointMaze Giant, Visual AntMaze Giant, and Visual Scene, which are representative OGBench tasks with limited fields of view.
We also evaluate on Car Racing and LabMaze, which similarly involve partial observability and restricted visual observations. We visualize the five evaluation environments in~\Cref{fig:eval_tasks}.
We further evaluate \algname in a more complex and complex 3D robotics benchmark by providing experiments on RoboCasa NavigateKitchen tasks.
Across NavigationKirchen we additionally control difficulty and occlusion level to further analyze benefits of \algname's planning capabilities under uncertainty.

\subsection{Baselines}

\p{Goal-Conditioned Reinforcement Learning} We evaluate against standard goal-conditioned RL approaches: GCIVL~\citep{ogbench_park2025} (Goal-Conditioned Implicit Value Learning) and GCIQL~\citep{kostrikov2021offline} (Goal-Conditioned Implicit Q-Learning), which learn value functions for goal-reaching tasks.
\p{World Model based Planning} We compare with recent world model approaches including 
PLDM~\citep{sobal2022joint} (Planning with Latent Dynamics Models), a latent dynamics approach for control; DINO-WM~\cite{zhou2025dinowm}, a deterministic latent world model on frozen DINOv2 features and LeWorldModel \citep{maes2026leworldmodel}, a JEPA-style world model that learns predictive representations for planning in latent space. 

\subsection{Implementation Details}


\p{Token Sequence Construction for Causal Attention} 
The velocity field $v_{\phi}$ is implemented as a transformer operating over a length-$(H+K)$ token sequence constructed by concatenating history and noised future embeddings: $x_{1:(H+K)} = [z_{\text{hist}}^{(1)}, \ldots, z_{\text{hist}}^{(H)}, z_\tau^{(1)}, \ldots, z_\tau^{(K)}]$, augmented with learnable positional embeddings. We employ causal attention masking, ensuring each future token at position $H+i$, where $i \in \{1,2,...,K\}$ can only attend to the $H$ history tokens and earlier future tokens, preventing information leakage from "unproduced" future states.

\begin{figure}
\vspace{-12pt}
    \centering
    \includegraphics[width=\linewidth]{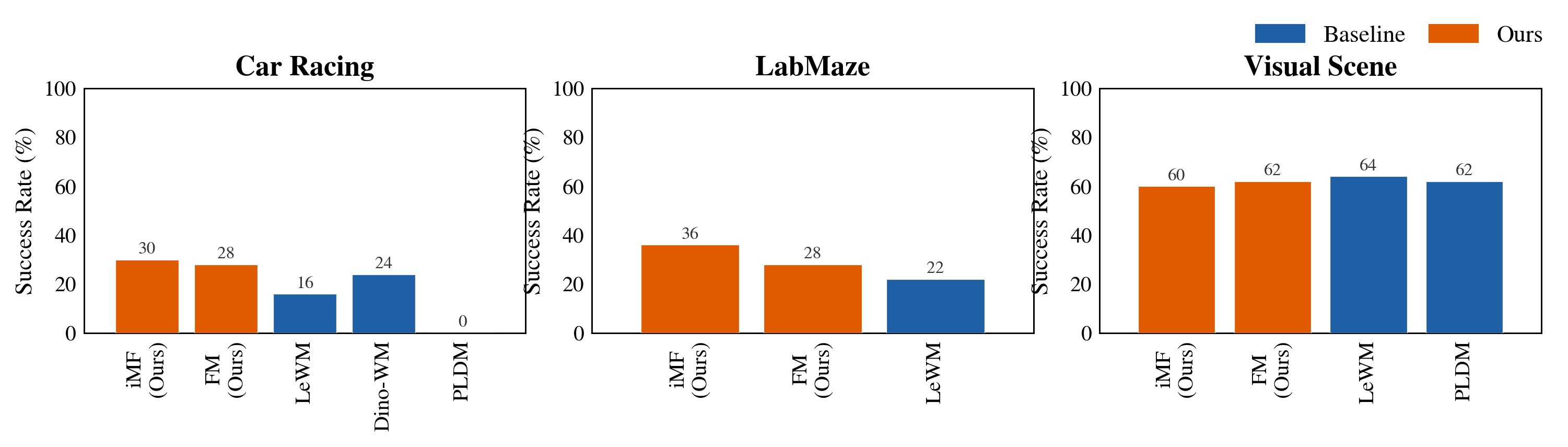}
    \caption{\textbf{Success rate comparison across benchmark environments. } \algname achieves competitive planning success rate across diverse navigation task under partially observable scenes and is on-par for complex manipulation in the Visual Scene.}
    \label{fig:success_2}
\end{figure}

\begin{figure}[t]
\vspace{-8pt}
\centering
\begin{subfigure}{0.32\textwidth}
\begin{tikzpicture}
\begin{axis}[
    width=\textwidth,
    height=5.5cm,
    xlabel={Disagreement Weight ($\beta$)},
    ylabel={Success Rate (\%)},
    xlabel style={font=\small},
    ylabel style={font=\small},
    tick label style={font=\footnotesize},
    legend pos=north east,
    legend style={
        font=\tiny, 
        rounded corners=1pt, 
        fill=white, 
        fill opacity=0.9,
        draw=black!20,
        inner sep=1.5pt,
        outer sep=1.5pt
    },
    axis lines=left,
    axis line style={line width=1pt},
    xmin=0.05, xmax=1.05,
    ymin=50, ymax=88,
    xtick={0.1, 0.4, 1.0},
    xticklabels={0.1, 0.4, 1.0},
    ytick={55, 65, 75, 85},
    mark size=3pt,
    line width=1.5pt,
]

\addplot[
    color=teal, 
    mark=*,
    mark options={fill=white, line width=1pt}
] coordinates {
    (0.1, 54.0) (0.4, 60.0) (1.0, 60.0)
};

\addplot[
    color=orange, 
    mark=square*,
    mark options={fill=white, line width=1pt}
] coordinates {
    (0.1, 72.0) (0.4, 84.0) (1.0, 80.0)
};

\addplot[
    color=purple, 
    mark=diamond*,
    mark options={fill=white, line width=1pt}
] coordinates {
    (0.1, 80.0) (0.4, 84.0) (1.0, 70.0)
};

\legend{NFE=1, NFE=2, NFE=4}
\end{axis}
\end{tikzpicture}
\caption{PointMaze Giant}

\end{subfigure}
\hfill
\begin{subfigure}{0.32\textwidth}
\begin{tikzpicture}
\begin{axis}[
    width=\textwidth,
    height=5.5cm,
    xlabel={Disagreement Weight ($\beta$)},
    ylabel={Success Rate (\%)},
    xlabel style={font=\small},
    ylabel style={font=\small},
    tick label style={font=\footnotesize},
    legend pos=north east,
    legend style={
        font=\tiny, 
        rounded corners=1pt, 
        fill=white, 
        fill opacity=0.9,
        draw=black!20,
        inner sep=1.5pt,
        outer sep=1.5pt
    },
    axis lines=left,
    axis line style={line width=1pt},
    xmin=0.05, xmax=1.05,
    ymin=22, ymax=30,
    xtick={0.1, 0.4, 1.0},
    xticklabels={0.1, 0.4, 1.0},
    ytick={23, 25, 27, 29},
    mark size=3pt,
    line width=1.5pt,
]

\addplot[
    color=teal, 
    mark=*,
    mark options={fill=white, line width=1pt}
] coordinates {
    (0.1, 26.0) (0.4, 24.0) (1.0, 28.0)
};

\addplot[
    color=orange, 
    mark=square*,
    mark options={fill=white, line width=1pt}
] coordinates {
    (0.1, 26.0) (0.4, 28.0) (1.0, 24.0)
};

\addplot[
    color=purple, 
    mark=diamond*,
    mark options={fill=white, line width=1pt}
] coordinates {
    (0.1, 26.0) (0.4, 26.0) (1.0, 26.0)
};

\legend{NFE=1, NFE=2, NFE=4}
\end{axis}
\end{tikzpicture}
\caption{AntMaze Giant}
\end{subfigure}
\hfill
\begin{subfigure}{0.32\textwidth}
\begin{tikzpicture}
\begin{axis}[
    width=\textwidth,
    height=5.5cm,
    xlabel={Disagreement Weight ($\beta$)},
    ylabel={Success Rate (\%)},
    xlabel style={font=\small},
    ylabel style={font=\small},
    tick label style={font=\footnotesize},
    legend pos=north east,
    legend style={
        font=\tiny, 
        rounded corners=1pt, 
        fill=white, 
        fill opacity=0.9,
        draw=black!20,
        inner sep=1.5pt,
        outer sep=1.5pt
    },
    axis lines=left,
    axis line style={line width=1pt},
    xmin=0.05, xmax=1.05,
    ymin=42, ymax=62,
    xtick={0.1, 0.4, 1.0},
    xticklabels={0.1, 0.4, 1.0},
    ytick={45, 50, 55, 60},
    mark size=3pt,
    line width=1.5pt,
]

\addplot[
    color=teal, 
    mark=*,
    mark options={fill=white, line width=1pt}
] coordinates {
    (0.1, 54.0) (0.4, 50.0) (1.0, 60.0)
};

\addplot[
    color=orange, 
    mark=square*,
    mark options={fill=white, line width=1pt}
] coordinates {
    (0.1, 44.0) (0.4, 48.0) (1.0, 50.0)
};

\addplot[
    color=purple, 
    mark=diamond*,
    mark options={fill=white, line width=1pt}
] coordinates {
    (0.1, 48.0) (0.4, 52.0) (1.0, 50.0)
};

\legend{NFE=1, NFE=2, NFE=4}
\end{axis}
\end{tikzpicture}
\caption{Visual Scene}
\end{subfigure}

\caption{\textbf{Ablation study on disagreement weight across environments.} \textbf{(a)} Visual PointMaze Giant shows clear sensitivity to hyperparameters with optimal performance at $\beta=0.4$ and Number Flow Evolutions(NFE)=2. \textbf{(b)} AntMaze Giant demonstrates more stable performance across settings, with modest gains from moderate exploration. \textbf{(c)} Visual Scene shows NFE values (NFE=1) achieve better performance.}
\vspace{-8pt}
\label{fig:ablation_comparison}
\end{figure}
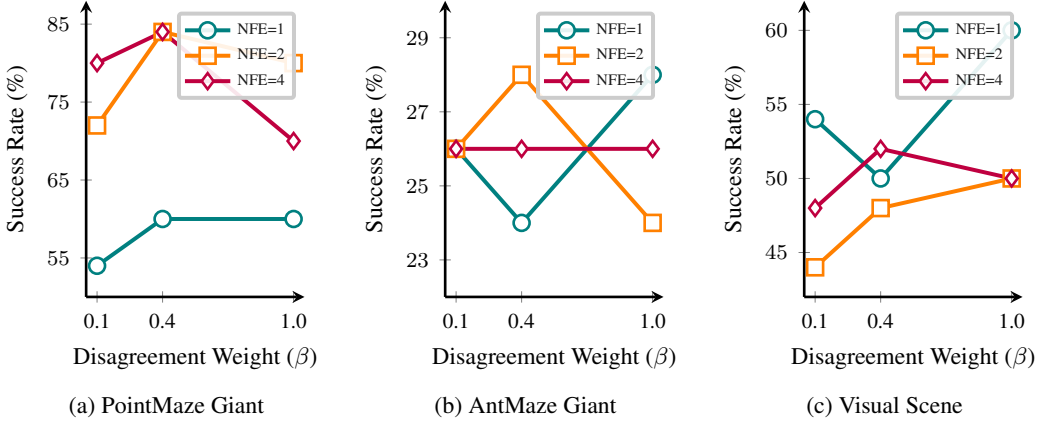

\p{Conditioning Architecture} For history tokens, conditioning combines action embeddings with temporal information: $c_{ctx} = c_{t} + t_{tok}$, where $c_{t}$ is the action embedding defined in the earlier section and $t_{tok}$ is the temporal embedding from TimeMLP($\tau$). Future tokens receive only temporal conditioning: $c_{fut} = t_{tok}$. The full conditioning sequence is constructed as $c_{1:(H+K)} = [c_{\text{ctx}}, c_{\text{fut}}]$ and passed to the transformer alongside the token sequence.


\p{Inference}
For each action chunk, we sample 10 noise samples. For each noise sample, we require 2-NFE in case of average velocity or 10-NFE in case of FM. All models are trained with three seeds (0, 1, 2) and each is evaluated with three evaluation seeds; we report mean and variance. During evaluation we evaluate 50 test scenarios for all methods and seeds in parallel on a single A6000 GPU with 48GB of GPU-RAM showing the efficiency of latent world models. The typical goal offset distance from the start planning place is $25$ steps.

$\beta$ and NFE are selected by grid search on a held-out set of scenes disjoint from training and test scenes. N, K and NFE are fixed across all main experiments.

\p{Training Details} All models were trained on A6000 GPUs. Especially with the iFM objective achieved early convergence after less than 5 GPU hours for Visual PointMaze data with simple action spaces and up to 20 GPU hours for Visual AntMaze with higher dimensional actions. 
All other environments land in-between. 
Training data across all OGBench~\cite{ogbench_park2025} tasks and LabMaze~\cite{beattie2016deepmind} tasks was produced from heuristic experts trajectories. 
RoboCasa experiments were trained on 500 human demonstrations of NavigationKitchen.

\subsection{Planning Performance of \algname}
In~\Cref{fig:success}, we compare our \algname with prior goal-conditioned and world-model baselines on two large-scale partially observable navigation tasks. On Visual PointMaze Giant, \algname achieves the best  performance across all methods, especially we have $\approx 22\%$ more success rate than LeWorldModel. IMF reaches a success rate of $84\% \pm 3.6$, while FM further improves to $86\% \pm 2.6$ vs. LeWM $64.5 \pm 2.9$ and DINO-WM $67.8 \pm 1.8$. These results substantially outperform prior model-based baselines such as LeWM, Dino-WM and PLDM, as well as goal-conditioned policy baselines such as GCIQL and GCIVL. This result suggests that explore-driven latent planing is more effective in visually constrained maze environments.

On Visual AntMaze Giant, the task becomes more challenging due to higher-dimensional action spaces. In this setting, FM achieves 30\% success and iMF achieves 28\%, achieving on-par performance with LeWM, DINO-WM and PLDM. 
Across both environments, the long-horizon setting, which doubles the target planning horizon, consistently reduces performance.

We further evaluate \algname on Car Racing, LabMaze, and Visual Scene, which also involve partial observability and limited visual context. Planning success rate results are summarized in~\Cref{fig:success_2}. Consistent with the navigation results, \algname performs strongly on Car Racing, where iMF and FM achieve 30\% and 28\% success, respectively, almost doubled the success rate of LeWM (16.0\%) and above DINO-WM (24.0\%). PLDM fails out-of-regime. On LabMaze, iMF again achieves the best performance. On Visual Scene, our methods remain competitive, though the smaller gap suggests that flow-based latent prediction is most effective in partially observable tasks with relatively low-dimensional action spaces.

The gain is not explained by inference compute: giving LeWM 6000 CEM candidates (20× default) raises PointMaze success only to 74.0\% (App. C.2), still $\approx10$ points below iMF; on AntMaze extra samples do not help (24.0–26.5\%).
\begin{figure}
    \centering
    \includegraphics[width=\linewidth]{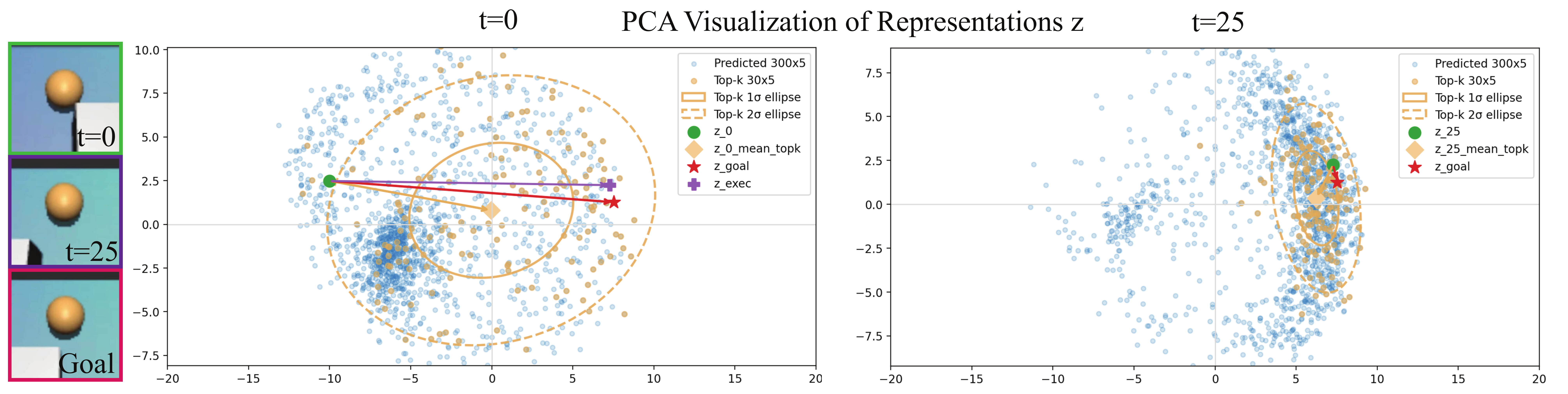}
    \caption{\textbf{Latent uncertainty contraction during planning on a single rollout.} Each panel shows the top-$k$ latent predictions in a shared PCA space, together with the current latent ($\mathbf z_{0}$ or $\mathbf z_{25}$), the fixed goal latent ($\mathbf z_\text{goal}$), and the top-k covariance ellipse. The red arrow measures $\|z_t - z_{goal} \|$}
    \label{fig:latent_analysis}
\end{figure}


\subsection{Planning Performance Under Controlled Uncertainty Levels in RoboCasa}

We further evaluate \algname on the RoboCasa NavigateKitchen task under controlled levels of uncertainty across general task difficulty and goal observability. We vary the robot's initial distance to the goal, which is moderately correlated with initial goal visibility ($\rho=-0.53$). Across this sweep, the goal is often partially visible or occluded at the start. Figure~\ref{fig:robocasa_rollouts} shows representative EpicWorldModel rollouts
in which the fridge is initially outside the robot's field of view or occluded
by kitchen fixtures, and the robot must explore before the goal becomes visible. We compare the deterministic LeWM baseline against our stochastic predictor with exploration disabled ($\beta=0$) and enabled ($\beta=0.1$, the default weight, not tuned for this task) and LeWM with 3× more CEM candidates (LeWM 3×) to control for sampling budget. We evaluate 15 episodes at each distance, for 60 episodes per method. For RoboCasa all models (LeWM and \algname variants) share a frozen DINOv3-ViT-S/16 encoder (App. B)

\p{Controlled Difficulty} Three trends emerge from experiments documented in~\Cref{tab:robocasa_controlled_difficulty}. First, under matched planning compute (the same CEM search budget), the stochastic predictor without exploration already substantially outperforms deterministic LeWM, achieving $71.7\%$ versus $45.0\%$ success rate when pooled across distances. Second, this gain cannot be explained solely by a larger sampling budget: increasing LeWM to $3\times$ more candidates improves short-range performance but provides little benefit at larger distances. Third, the exploration bonus becomes particularly useful when the goal cost is less informative. At 1.26 m, enabling exploration increases iMF success from $46.7\%$ to $80.0\%$, while leaving easier short-range settings unchanged. The FM variant performs best at the largest distance, reaching $93.3\%$ success compared with $20.0\%$ for LeWM.

\p{Controlled Occlusions/Visibility} Distance, however, is only an indirect proxy for partial observability. We therefore additionally measure the fraction of the goal visible from the robot's initial observation. Specifically, we sample approximately 200 points on the goal fixture geometry and compute the fraction that lies within the camera frustum and is not occluded by scene geometry. The resulting score ranges from 0 (fully hidden or outside the field of view) to 1 (fully visible and unobstructed). This geometric measure requires no learned model or tuned threshold and is additionally verified using segmentation masks and manual inspection.

The occlusion-controlled RoboCasa experiments in~\Cref{tab:robocasa_occlusion_visibility} provide preliminary evidence that uncertainty-driven exploration yields information about the hidden goal location. When the goal is initially occluded ($<0.5$), exploration substantially improves success ($1/6 \rightarrow 5/6$). This benefit decreases as visibility increases and can even reduce goal-reaching performance when the goal is fully observable ($\geq 0.95$; $21/22 \rightarrow 19/22$). These results suggest that the variance bonus is useful for goal discovery under partial observability but unnecessary once the goal is visible.

\begin{table}[tbp]
\caption{\textbf{Controlled difficulty analysis on RoboCasa NavigateKitchen.}
Success rate (\%) across increasing start distances. Best results are shown in \textbf{bold}, and the second-best average is \underline{underlined}. (iMF NFE=2; FM NFE=5; LeWM 3× n = 8–12)}
\vspace{6pt}
\centering
\small
\setlength{\tabcolsep}{10pt}
\renewcommand{\arraystretch}{1.08}
\begin{tabular}{lccccc}
\toprule
\textbf{Start Dist.} &
\textbf{LeWM} &
\textbf{LeWM $3\times$} &
\textbf{iMF ($\beta=0$)} &
\textbf{iMF ($\beta=.1$)} &
\textbf{FM ($\beta=.1$)} \\
\midrule
0.50 m & 60.0 & 93.3 & \textbf{100.0} & \textbf{100.0} & \textbf{100.0} \\
0.92 m & 46.7 & 75.0 & 80.0 & 80.0 & \textbf{86.7} \\
1.26 m & 53.3 & 40.0 & 46.7 & \textbf{80.0} & \textbf{80.0} \\
1.46 m & 20.0 & 25.0 & 60.0 & 60.0 & \textbf{93.3} \\
\midrule
\textbf{Average} &
45.0 &
58.3 &
71.7 &
\underline{80.0} &
\textbf{90.0} \\
\bottomrule
\end{tabular}
\label{tab:robocasa_controlled_difficulty}
\end{table}
\begin{figure}[tbp]
  \centering
  \includegraphics[width=\textwidth]{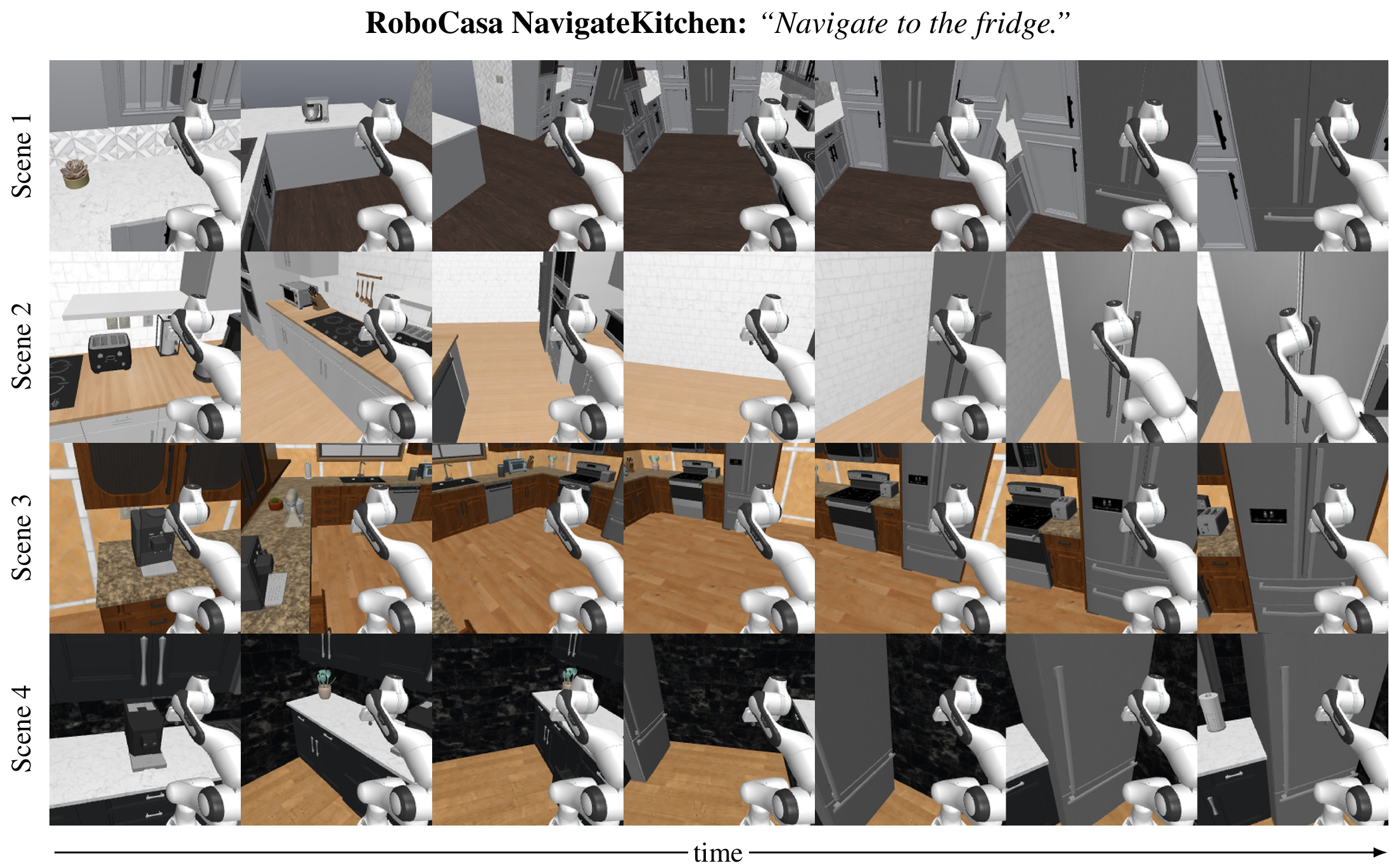}
  \caption{\textbf{Qualitative rollouts of EpicWorldModel on RoboCasa NavigateKitchen.}
  Each row shows an episode in a different kitchen layout from the validation scene split.
  The initial visibility is < 0.5 for all episodes, only displaying the hardest tasks.
  All episodes are conditioned on the same instruction, and frames are ordered left to right in time.
  At the start of each episode the fridge is outside the field of view or occluded by cabinets and counters. 
  Guided by the uncertainty cost $\mathcal{U}$, the robot first turns and moves toward unobserved regions of the kitchen, then commits to the goal once the
  fridge becomes visible.
  All shown results are from an iMF model ($\beta = 0.1$, NFE $= 2$).}
  \label{fig:robocasa_rollouts}
\end{figure}
\begin{table}[t]
\caption{\textbf{Occlusion/visibility analysis on RoboCasa NavigateKitchen.}
Successes / episodes (pooled \%) stratified by initial goal visibility. Lower visibility indicates stronger partial observability. Best results are shown in \textbf{bold}.  (iMF NFE=2; FM NFE=5; LeWM 3× n = 8–12)}
\vspace{6pt}
\centering
\small
\setlength{\tabcolsep}{10pt}
\renewcommand{\arraystretch}{1.08}
\begin{tabular}{lcccc}
\toprule
\textbf{Initial Visibility} &
\textbf{LeWM} &
\textbf{iMF ($\beta=0$)} &
\textbf{iMF ($\beta=.1$)} &
\textbf{$n$} \\
\midrule
$<0.5$ &
0/6 &
1/6 &
\textbf{5/6} &
6 \\

$0.5$--$0.8$ &
2/7 &
5/7 &
\textbf{6/7} &
7 \\

$0.8$--$0.95$ &
1/4 &
\textbf{2/4} &
1/4 &
4 \\

$\geq0.95$ &
10/22 &
\textbf{21/22} &
19/22 &
22 \\
\midrule
\textbf{Pooled} &
13/39 (33.3) &
29/39 (74.4) &
\textbf{31/39 (79.5)} &
39 \\
\bottomrule
\end{tabular}
\label{tab:robocasa_occlusion_visibility}
\end{table}

\subsection{Ablation Study on the choice of \texorpdfstring{$\beta$}{beta} and the Number of Flow Evaluations (NFE)}
We conduct an ablation study to understand the sensitivity of our approach to key hyperparameters: the disagreement weight $\beta$, which controls the exploration-exploitation trade-off, and the Number of Flow Evaluations (NFE) during planning, which determines the computational budget for inference.

We evaluate three representative environments from our benchmark suite: Visual PointMaze Giant, Visual AntMaze Giant, and Visual Scene. For each environment, we conduct an ablation over the disagreement weight $\beta \in \{0.1, 0.4, 1.0\}$ and the number of flow evaluations $\mathrm{NFE} \in \{1,2,4\}$.~\Cref{fig:ablation_comparison}~a--c summarize the sensitivity of performance to these two hyperparameters across environments. The results indicate that hyperparameter selection should be adapted to the structure and dimensionality of each task. Across most environments, $\beta=0.4$ and $\mathrm{NFE}=2$ achieve a favorable trade-off between uncertainty-driven exploration and computational cost. In contrast, the Visual Scene results suggest that, in high-dimensional action spaces, single-step prediction combined with a larger exploration weight may provide a more effective planning strategy. Automatically adapting $\beta$ is left to future work.

\subsection{Analysis on Exploration Behavior in Latent Space Planning}\label{sec:exp_behavior}
\Cref{fig:latent_analysis} provides a qualitative view of the top-$k$ and the full cross entropy samples' latent predictions for a single test case. The full cross-entropy sample predictions are shown in blue and the top-k selections in yellow. We project the $D = 192$ latent embeddings $\mathbf z_{t}$ to the PCA 2-Dimensional space for analysis and visualization. At the first planning point ($t = 0$), the elite latent cloud is broad, reflecting substantial uncertainty over future outcomes when the agent is still far from the goal. After executing the first plan, the executed rollout reaches a latent state close to the goal within the first 25-step plan. By the second planning point ($t = 25$), the elite cloud has contracted markedly around the current and goal-adjacent region, with the top-k ellipse area shrinking from 60.99 to 11.86. This example is consistent with the hypothesis that predictive uncertainty is larger early in planning and decreases once the agent approaches the goal. Furthermore, the observed mean of the top-k predictions points in the direction toward the goal, suggesting that the predictive variance of the stochastic predictor is consistent with goal-directed progress rather than random noise.



\section{Conclusion}
We present \algname, an uncertainty-aware planning framework leveraging stochastic JEPAs based on flow matching. Through theoretical motivation and experimental validation, we show that \algname is a stable JEPA capable of achieving high planning success rates, especially in environments under partial observability. Although our results demonstrate the promise of exploration-driven planning with latent world models, several limitations remain. In particular, goal-image conditioning can be restrictive, as many real-world tasks are more naturally specified through language and goal images may provide privileged information unavailable at deployment. Extending the framework to align visual observations with language-based or abstract goal specifications is therefore an important direction for future work.



\clearpage

\bibliographystyle{unsrtnat}
\bibliography{references.bib}

\clearpage

\appendix
\section{Proofs}
\label{app:proofs}

\propboundedgradient*
\begin{proof}
By the law of total variance applied to the conditional velocity target $v_\text{tgt} = \mathbf{z}^{(1)} - \mathbf{z}^{(0)}$:
\begin{equation}
    \mathrm{Var}(v_\text{tgt}) = \mathbb{E}_{\mathbf{z}^{(\tau)}}[\mathrm{Var}(v_\text{tgt} \mid \mathbf{z}^{(\tau)})]
    + \mathrm{Var}(\mathbb{E}[v_\text{tgt} \mid \mathbf{z}^{(\tau)}]).
\end{equation}
On the left hand side, by independence of $\mathbf{z}^{(0)}$ and $\mathbf{z}^{(1)}$, we have:
\begin{equation}
    \mathrm{Var}(v_\text{tgt}) = \mathrm{Var}(\mathbf{z}^{(1)} - \mathbf{z}^{(0)}) = \sigma^2 I_d + I_d = (\sigma^2 + 1)I_d,
\end{equation}
where we denote the variance of the SIGReg-induced latent space by $\sigma^2$. On the right hand side, $\mathbb{E}[v_\text{tgt} \mid \mathbf{z}^{(\tau)}] = v_\text{tgt}(\mathbf{z}^{(\tau)})$ is the marginal vector field~\cite{lipman2022flow}. Since $\mathbf{z}^{(\tau)} = (1-\tau)\mathbf{z}^{(0)} + s\mathbf{z}^{(1)}$, the joint covariance of $(\mathbf{z}^{(1)} - \mathbf{z}^{(0)},\, \mathbf{z}^{(\tau)})$ is given by:
\begin{equation}
    \mathrm{Cov}(\mathbf{z}^{(1)} - \mathbf{z}^{(0)},\, \mathbf{z}^{(\tau)})
    = \tau \,\mathrm{Var}(\mathbf{z}^{(1)}) - (1-\tau)\,\mathrm{Var}(\mathbf{z}^{(0)})
    = (\tau\sigma^2 - (1-\tau))I_d,
\end{equation}
\begin{equation}
    \mathrm{Var}(\mathbf{z}^{(\tau)}) = (1-\tau)^2 I_d + \tau^2\sigma^2 I_d
    = [(1-\tau)^2 + \tau^2\sigma^2]I_d.
\end{equation}
The marginal vector field is the linear conditional expectation:
\begin{equation}
    v_\text{tgt}(\mathbf{z}^{(\tau)}) = \frac{\tau\sigma^2 - (1-\tau)}{(1-\tau)^2 + \tau^2\sigma^2}\, \mathbf{z}^{(\tau)},
\end{equation}
with variance:
\begin{equation}
    \mathrm{Var}(v_\text{tgt}(\mathbf{z}^{(\tau)}))
    = \frac{(\tau\sigma^2 - (1-\tau))^2}{(1-\tau)^2 + \tau^2\sigma^2}I_d.
\end{equation}
Subtracting:
\begin{align}
    \mathbb{E}_{\mathbf{z}^{(\tau)}}[\mathrm{Var}(v_\text{tgt} \mid \mathbf{z}^{(\tau)})] &= 
    \mathrm{Var}(v_\text{tgt}) - \mathrm{Var}(\mathbb{E}[v_\text{tgt} \mid \mathbf{z}^{(\tau)}])\\
    &= (\sigma^2 + 1)I_d -  \frac{(\tau\sigma^2 - (1-\tau))^2}{(1-\tau)^2 + \tau^2\sigma^2}I_d \\
    &= \frac{\sigma^2}{(1-\tau)^2 + \tau^2\sigma^2}I_d.
\end{align}

Since $\mathbf{z}^{(0)}\sim\mathcal{N}(0,I_d)$, $\mathbf{z}^{(1)}\sim\mathcal{N}(0,\sigma^2
I_d)$, and $\mathbf{z}^{(\tau)}=(1-\tau)\mathbf{z}^{(0)}+s\mathbf{z}^{(1)}$, the triple
$(\mathbf{z}^{(0)},\mathbf{z}^{(1)},\mathbf{z}^{(\tau)})$ is jointly Gaussian. As a result, conditional variances are homoscedastic.
Therefore, for each flow time $s$:
\begin{equation}
    \mathrm{Var}(v_\text{tgt} \mid \mathbf{z}^{(\tau)}) = \frac{\sigma^2}{(1-\tau)^2 + \tau^2\sigma^2}I_d.
\end{equation}
Under SIGReg, $\sigma^2 \approx 1$, yielding the desired result.
\end{proof}

\propuncertainty*

\begin{proof}
\p{(i) Unbiasedness}
For a well-trained $\mathcal{P}_\varphi$, the ensemble $\{{\hat{\mathbf{z}}}{^{(1)}}[i]\}_{i=1}^N$ are i.i.d.\ draws from $p(\mathbf{z}^{(1)}\mid\mathbf{x}_t,\mathbf {c}_t)$. The sample variance $\widehat{\mathrm{Var}}_i([{\hat{\mathbf{z}}}{^{(1)}}[i]]_j)$ is an unbiased estimator of the $j$-th marginal variance $[\Sigma_{\mathbf {c}_t}]_{jj}$ for each $j$. Averaging over dimensions:
\begin{equation}
    \mathbb{E}[\mathcal{U}]
    = \frac{1}{d}\sum_{j=1}^d [\Sigma_{\mathbf {c}_t}]_{jj}
    = \frac{\mathrm{tr}(\Sigma_{\mathbf {c}_t})}{d}.
\end{equation}

\p{(ii) Entropy upper bound}
By the maximum entropy theorem, given the same covariance matrix, the Gaussian distribution maximizes differential entropy:
\begin{equation}
    H(\mathbf{z}^{(1)}\mid\mathbf{x}_t,\mathbf {c}_t)
    \;\leq\; \tfrac{1}{2}\log\bigl((2\pi e)^d\,|\Sigma_{\mathbf {c}_t}|\bigr).
\end{equation}
This bound involves the determinant $|\Sigma_{\mathbf {c}_t}|$, which is not directly observable from the ensemble. We relate it to $\mathcal{U}$ via AM-GM applied to the eigenvalues $\{\lambda_i\}_{i=1}^d$ of $\Sigma_{\mathbf {c}_t}$:
\begin{equation}
    \frac{1}{d}\log|\Sigma_{\mathbf {c}_t}|
    = \frac{1}{d}\sum_{j=1}^d \log\lambda_j
    \;\leq\; \log\!\left(\frac{1}{d}\sum_{i=1}^d \lambda_i\right)
    = \log\frac{\mathrm{tr}(\Sigma_{\mathbf {c}_t})}{d}
    = \log\mathcal{U},
\end{equation}
where we applied Jensen's inequality to the concave $\log$, with equality if and only if $\Sigma_{\mathbf {c}_t}\propto I_d$. Therefore, $|\Sigma_{\mathbf {c}_t}|\leq\mathcal{U}^d$, and we obtain:
\begin{equation}
    H(\mathbf{z}^{(1)}\mid\mathbf{x}_t,{\mathbf {c}_t})
    \;\leq\; \tfrac{1}{2}\log\bigl((2\pi e\,\mathcal{U})^d\bigr)
    = \tfrac{d}{2}\log(2\pi e\cdot\mathcal{U}).
\end{equation}

\p{(iii) Bounded scale}
By the law of total variance:
\begin{equation}
    \mathrm{Cov}(\mathbf{z}^{(1)})
    = \mathbb{E}_{\mathbf{x}_t,{\mathbf {c}_t}}[\Sigma_{\mathbf {c}_t}]
    + \mathrm{Cov}(\mathbb{E}[\mathbf{z}^{(1)}\mid\mathbf{x}_t,{\mathbf {c}_t}]).
\end{equation}
Given the isotropic Gaussian latent space, $\mathbb{E}[\mathrm{tr}(\Sigma_{\mathbf {c}_t})] \leq
\mathrm{tr}(\mathrm{Cov}(\mathbf{z}^{(1)})) \approx d$. Therefore, $\mathbb{E}[\mathcal{U}] \lesssim 1$, with equality when $\mathrm{Cov}(\mathbb{E}[\mathbf{z}^{(1)}\mid\mathbf{x}_t,{\mathbf {c}_t}])
\approx 0$, i.e.\ when conditioning is uninformative.
\end{proof}

\section{Additional Implementation Details}

\p{Latent Space Architecture} \algname operates in a learned latent space rather than directly on high-dimensional observations. Input observations $o_{t} \in \mathbb{R}^{3\times224\times224}$ are processed by $\mathcal{E}_\theta$, which has a Vision Transformer (ViT-Tiny) backbone with patch size 14 and embedding dimension $D_{e} = 192$, and trained from scratch without pretraining. The CLS token is projected through a two-layer MLP with BatchNorm to produce latent embeddings $z \in \mathbb{R}^{D}$, where $D = 192$. Goal states $z_{goal}$ are encoded using the same encoder $\mathcal{E}_\theta$ to ensure goal distance computes within the shared latent space.

\p{Action Encoding} Actions are embedded via a Conv1d-MLP architecture operating on frame-skipped action vectors $a \in \mathbb{R}^{F \cdot d_{a}} $ with frameskip $F = 5$. For example, PointMaze environments with $d_{a} = 2$ produce $10$-dimensional action inputs, which are embedded to generate one $D$-dimensional action token per predicted future step.

\p{Output Slicing} The transformer processes both history and future tokens but only returns predictions for the future portion: the output is sliced to out$[:,K :]$, corresponding to the predicted velocity field for $K$ future time steps.

\p{Planner} LeWM~\cite{maes2026leworldmodel}, iMF and FM share a CEM planner over 25-step action sequences (5 chunks $\times$ 5 actions). 
Each planning call runs 30 iterations with 300 candidates, and a selection of top-30 sequences, initialised with std equal to the action scale. The final best mean is executed open-loop before replanning. 
We do 2 planning calls per 50-step episode and 3 planning calls for the long-horizon experiments. LeWM scores each candidate with one deterministic rollout using the final-step latent MSE to the latent goal $z_g$. 
\algname draws N rollouts (FM: 10 NFE; iMF: 2 NFE) and scores mean goal $MSE - \beta U$, both z-score normalised across candidates.

\p{Implementation and training details for RoboCasa task} All models (LeWM and \algname) are trained on 500 human demonstrations of NavigateKitchen (72,786 frames; held-out ``target'' scene split), with a frozen DINOv3-ViT-S/16 encoder (mean-pooled patch tokens, 384-d; zero encoder
updates), followed by a trained projector and latent-dynamics predictor,using the JEPA-style latent-prediction recipe (predicting future projected latents conditioned on actions; SIGReg regularization; no pixel reconstruction). Inputs consist of a 3-frame history with a frameskip of 5 (single egocentric camera, $224 \times 224$); actions are 3-DoF base velocity commands (the arm is masked for this navigation task). Optimization uses AdamW ($\mathrm{lr}=5 \times 10^{-5}$, $\mathrm{wd}=10^{-3}$), with a batch size of 448, bf16 precision, and 100 epochs on a single GPU ($\sim$2.5--4 h per model). The three variants---deterministic AR, one-step iMeanFlow (iMF), and 10-step flow matching (FM)---differ only in the predictor family; the encoder, projector architecture, data, and all hyperparameters are identical. All evaluations use each model's final (epoch-100) checkpoint; training losses converged well before epoch 100.

\section{Additional Results}
\label{app:additional_results}

We also report additional experiments that complement the main results including:

\begin{itemize}
    \item An ablation disabling the exploration term $\mathcal{U}$
    \item a compute-matched comparison (in contrast to a step-matched) against the deterministic baseline
    \item a sweep over sampled futures
    \item and a population-level check of the latent-space analysis in Sec.~\ref{sec:exp_behavior}.
\end{itemize}

Unless noted otherwise, all experiments follow the CEM/MPC planning protocol of App.~B as used in the main paper. 

\subsection{Effect of the Exploration Term (\texorpdfstring{$\beta = 0$}{beta = 0})}
\label{app:c1_beta_zero}

To isolate the contribution of the exploration term from the contribution of the predictor itself, we extend $\beta$/NFE sweep of Fig.~6(a) with a $\beta = 0$ condition on Visual PointMaze Giant.
Disabling the exploration term lowers success by 7 points for iMF (NFE~=~4, $\beta = 0.4$) and by 6 points for FM at NFE~=~10 ($\beta = 0.4 $).
The stochastic predictor alone therefore explains part, but not all of the gain over the deterministic LeWM~\cite{maes2026leworldmodel}.
This can only be attributed to the exploration term in the planning cost.

\begin{table}[ht]
\centering
\small
\caption{Effect of disabling the exploration term ($\beta=0$) on Visual
PointMaze Giant, as a function of the number of function evaluations (NFE)
at planning time. Left: iMeanFlow (iMF); right: flow matching (FM). Success rate (\%), single training/evaluation run per cell. Best value per column in \textbf{bold}.}
\vspace{4pt}
\label{tab:beta_zero}
\centering
\begin{tabular}{c | ccccc | cc}
\hline
 & \multicolumn{5}{c|}{\textbf{iMF}} & \multicolumn{2}{c}{\textbf{FM}} \\
\cline{2-8}
$\beta \backslash$ NFE & 1 & 2 & 4 & 8 & 10 & 4 & 10 \\
\hline
0.0 & 51.0 & 69.5 & 76.5 & 81.5 & 81.5 & 78.0 & 79.5 \\
0.4 & \textbf{51.5} & \textbf{74.5} & \textbf{83.5} & \textbf{86.0} & \textbf{86.0} & \textbf{81.0} & \textbf{85.5} \\
\hline
\end{tabular}
\end{table}

\subsection{Compute-Matched LeWorldModel}
\label{app:c2_compute_matched}

Our method evaluates $N$ stochastic rollouts per candidate at planning time, while LeWM, as a deterministic model, does only compute a single rollout per action candidate.
Therefore a fraction of \algname could be in principle come from the additional inference compute rather than the exploration cost term and stochastic predictor.
We test this by giving the deterministic LeWM baseline additional CEM action candidates to approximate the per-candidate cost of our $N \times \text{NFE} = 20$ (NFE=4, N=5)  evaluations up to $20 \times 300 = 6000$ candidates.

\Cref{tab:compute_matched} shows that scaling LeWM's action candidate budget alone recovers only part of the gap. On PointMaze the success rises from 54.0\% to 74.0\% but plateaus with any additional compute.
The remaining gap to \algname iMF with $\approx$10 points remains at lower and matched compute.
Similarly to the main findings AntMaze is not bounded by planning compute or methodology and rather is bounded by task and action-space complexity.
We refer to the discussions in Sec.~4.3 for the discussion on the AntMaze success rate.

\begin{table}[ht]
\centering
\small
\caption{Compute-matched comparison: LeWM success rate (\%) against
the number of CEM action candidates (default: 300), on Visual PointMaze Giant
and Visual AntMaze Giant. Bottom rows: EpicWorldModel (iMF, $\beta=0.4$)
at NFE~=~2 and NFE~=~4 for reference (from Tab.~\ref{tab:beta_zero}).
Best value per column in \textbf{bold}.}
\label{tab:compute_matched}
\begin{tabular}{l cc}
\toprule
CEM candidates & PointMaze & AntMaze \\
\midrule
300 (default)$^{\dagger}$ & 64.0 & 26.0 \\
600                 & 68.0 & \textbf{26.5} \\
900                 & 72.0 & 24.0 \\
1500                & 74.0 & \textbf{26.5} \\
6000                & 74.0 & 25.0 \\
\midrule
EpicWM (iMF, $\beta{=}0.4$, NFE{=}2) & 74.5 & --- \\
EpicWM (iMF, $\beta{=}0.4$, NFE{=}4) & \textbf{83.5} & --- \\
\bottomrule
\end{tabular}
\end{table}

\subsection{Number of Sampled Futures}
\label{app:c3_num_futures}

We further ablate how the number of sampled future states $N$ (distinct from the planning horizon $K$) affects planning performance.
In Tab.~\ref{tab:num_futures} we sweep $N \in \{2, 5, 10\}$ and at $\beta \in \{0, 0.4\}$ on the Visual PointMaze task for FM (NFE=10) and iMF (NFE=2) stochastic predictors.
Across all four settings the effect of $N$ is small, showing the robustness of the stochastic predictor.
Similar to the optimal selection of $\beta$ we see an automatic $N$ selection beyond a single default across tasks as future work.

\begin{table}[ht]
\centering
\small
\caption{Effect of the number of sampled futures $N$ on Visual PointMaze
Giant success rate (\%), for both predictor families (FM and iMF) at $\beta \in
\{0, 0.4\}$ (NFE fixed at each method's main-text default; App.~B). $^{*}$
marks $N{=}5$, the default value used in the main results. ``Range'' is
$\max - \min$ over the row. Best value per row in \textbf{bold}.}
\label{tab:num_futures}
\begin{tabular}{ll ccc c}
\toprule
Method & $\beta$ & $N=2$ & $N=5^{*}$ & $N=10$ & Range \\
\midrule
iMF & 0.0 & \textbf{78.0} & 76.5 & 77.5 & 1.5 \\
iMF & 0.4 & 80.0 & \textbf{83.5} & 78.0 & 5.5 \\
FM  & 0.0 & 79.0 & \textbf{82.0} & 81.0 & 3.0 \\
FM  & 0.4 & \textbf{81.5} & 81.0 & 81.0 & 0.5 \\
\bottomrule
\end{tabular}
\end{table}

\subsection{Latent-Space Analysis over All Evaluation Scenarios}
\label{app:c5_latent_analysis}

Section 4.6 shows an qualitative example of the exploration signal on a single rollout.
We show that the top-$k$ predictions in there latent space contract as the agent approaches the goal (Fig.~7; $1\sigma$ ellipse area shrinking from 60.99 to 11.86).
In addition to that, we provide a quantitative analysis over 50 evaluation scenarios.
Across all episodes the distance to the goal latent $\lVert z_t - z_{\mathrm{goal}} \rVert$ decreases from $17.4 \pm 5.1$ at $t{=}0$ to $10.4 \pm 5.8$ at $t{=}25$, corresponding to a median progress of 45.5\% (58.6\% on the 80\% of episodes that succeed).
Among all 1.500 sampled future predictions across these rollouts, the closest sample to the ground-truth future latent lies at a median distance of 3.5.
This is roughly five times smaller than the initial distance to the goal and further delivers a median cosine alignment of 0.88 between the predicted and ground-truth future directions. 
These results are consistent with the sampled futures tracking the ground-truth trajectory rather than producing arbitrary or uninformative spread, consistent with the exploration signal being useful for goal approach, as illustrated qualitatively in Fig.~7.


\end{document}